\documentclass[letterpaper, 10 pt, conference]{ieeeconf}

\IEEEoverridecommandlockouts                     
\usepackage{microtype}
\usepackage{amsmath, amssymb}
\usepackage{graphicx}
\usepackage{stfloats}
\usepackage[ruled,vlined,noend]{algorithm2e}
\usepackage{booktabs, multirow}
\usepackage{balance}
\usepackage[colorlinks]{hyperref}
\usepackage{color}
\usepackage{float}
\usepackage{gensymb}
\usepackage[usenames,dvipsnames]{xcolor}
\usepackage{soul}
\usepackage{bmathshortcuts}

\usepackage[noadjust]{cite}

\usepackage{pifont}

\newtheorem{theorem}{Theorem}

\newtheorem{definition}{Definition}

\title{\LARGE \bf
Force-Safe Environment Maps and Real-Time Detection for Soft Robot Manipulators}

\author{Akua K. Dickson$^{1}$, Juan C. Pacheco Garcia$^{2}$, Andrew P. Sabelhaus$^{1,2}$
\thanks{This work was in part supported by the U.S. National Science Foundation under Award No. 2340111, 2209783, and a Graduate Research Fellowship.}
\thanks{$^1$A.K. Dickson and A.P. Sabelhaus are with the Division of Systems Engineering, Boston University, Boston MA, USA. {\tt \footnotesize \{akuad, asabelha\}@bu.edu} }
\thanks{$^2$J.C. Pacheco Garcia, and A.P. Sabelhaus are with the Department of Mechanical Engineering, Boston University, Boston MA, USA. {\tt \footnotesize \{jcp29, asabelha\}@bu.edu} } 
}

\begin{document}
\pagestyle{empty}
\urlstyle{tt}
\maketitle
\thispagestyle{empty}
\begin{abstract}
Soft robot manipulators have the potential for deployment in delicate environments to perform complex manipulation tasks. However, existing obstacle detection and avoidance methods do not consider limits on the forces that manipulators may exert upon contact with delicate obstacles. This work introduces a framework that maps force safety criteria from task space (i.e. positions along the robot's body) to configuration space (i.e. the robot's joint angles) and enables real-time force safety detection.
We incorporate limits on allowable environmental contact forces for given task-space obstacles, and map them into configuration space (C-space) through the manipulator’s forward kinematics. This formulation ensures that configurations classified as safe are provably below the maximum force thresholds, thereby allowing us to determine force-safe configurations of the soft robot manipulator in real-time. We validate our approach in simulation and hardware experiments on a two-segment pneumatic soft robot manipulator. Results demonstrate that the proposed method accurately detects force safety during interactions with deformable obstacles, thereby laying the foundation for real-time safe planning of soft manipulators in delicate, cluttered environments.
\end{abstract}

\section{Introduction}
\label{sec:introduction}
Soft robot manipulators are widely known for their ability to interact compliantly with their environments to perform delicate tasks such as medical procedures \cite{majidi_2014_soft_robots, Burgner-Kahrs2015Continuum}. Their intrinsic compliance often results in lower contact forces compared to rigid counterparts \cite{abidi_intrinsic_2017}, which has led to the
common perception that soft robots are inherently safe. However, compliance alone may not guarantee safety. The magnitude of environmental contact forces during robot motion depends not only on the robot’s mechanical properties but also on its dynamics and the stiffness of the environment \cite{Wang2024Perceived}. For sensitive tasks such as medical procedures, it is desirable to limit these environmental contact forces below certain limits \cite{Burgner-Kahrs2015Continuum}. This motivates a shift from purely geometric collision detection and avoidance toward force safety, where safety is defined by limits on contact forces rather than by the absence of contact.

Collision detection and obstacle avoidance have been extensively studied for rigid robot manipulators  \cite{Wei2018AMethod,Haddadin2017Robot }. Many task-space planning algorithms neglect the robot’s full-body dynamics and often produce trajectories that are not dynamically feasible, particularly for systems with high degrees of freedom (DOF) \cite{Dmitry2011Task, Bounini2017Modified}. To address these limitations, motion planning is reformulated in the \textit{configuration space} \cite{Lozano1979Algorithm,Lozano-Perez1981Automatic, Lozano-Perez1983Spatial}, where each point represents a unique robot configuration. This transformation enables efficient safe motion planning using sampling based approaches such as probabilistic roadmaps (PRM) \cite{Chen2021Path, Kavraki1996Probablistic}, rapidly exploring random trees (RRTs) \cite{Yuan2020Heuristic,Wang2020Collision} and optimization-based approaches \cite{Dai2020Robust, Ratliff2009CHOMP}. While effective for collision-free motion planning of rigid robot manipulators, these methods do not directly extend to compliant soft manipulators. Furthermore, traditional collision detection treats safety as a purely geometric notion, i.e., the manipulator is considered safe as long as its body does not intersect any task space obstacle \cite{Hongyan2022Real-Time,Chakravarthy1998Obstacle}. As such, these approaches do not consider obstacle deformation and force safety criteria (Defn. 1 below) \cite{Dickson2025Safe}.
\begin{figure}[!t]
\centering     \includegraphics[width=1.0\columnwidth]{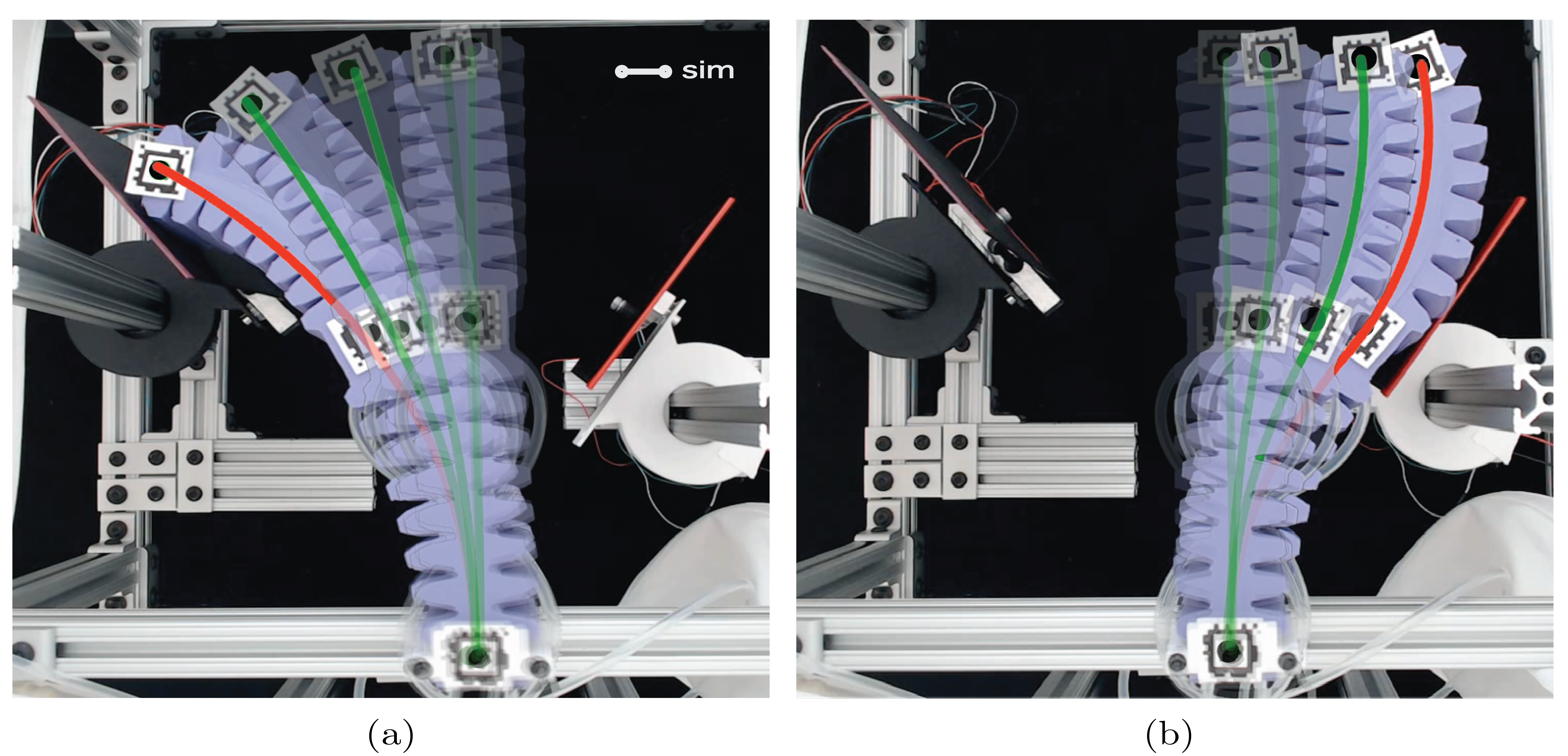}
\vspace{-0.5cm}
\caption{Soft robot manipulator motions in hardware and simulation are overlaid on top of each other, with colors indicating force safety: green for safe and red for unsafe configurations. (a) shows a timelapse of motion toward the first obstacle, and (b) toward the second obstacle.}
\label{fig:overview_diagram}
\vspace{-0.5cm}
\end{figure} 

In contrast, soft robot manipulators possess infinite or very high DOF, making both configuration space representation and collision reasoning far more challenging.
Existing work in soft robot motion planning often extends rigid-robot sampling-based methods such as RRT* \cite{Khan2020Control} and potential field methods \cite{Ataka2016Real-Time} under kinematic assumptions such as the piecewise constant-curvature approximation. Furthermore, task-space collision avoidance methods \cite{Liu2023Path} generate collision-free trajectories while considering kinematic constraints of soft manipulators. Configuration-space based collision detection and motion planning methods \cite{Veil2025Shape, Li2012Exact} on the other hand, further enhance accuracy and dynamic feasibility by leveraging the exact constant-curvature geometry and operating within the state space of the robot. While these methods enable collision detection and obstacle avoidance, they typically treat collisions as strictly undesirable. As a result, they do not provide guarantees on force safety or account for environmental contact forces. This manuscript bridges this gap by presenting a novel approach for real-time force-unsafe deformation detection for soft robot manipulators (Fig. \ref{fig:overview_diagram}). 
We contribute:
\begin{itemize}
    \item A method (Fig. \ref{fig:overviewfigure2}) to map force safety criteria for soft robot manipulators into the robot's configuration space.
    \item A method to detect force-safe states in real time.
    \item A validation in both simulation and hardware, demonstrating that our method detects force safety.
\end{itemize}

\begin{figure*}[!t]
    \centering \includegraphics[width=1\textwidth]{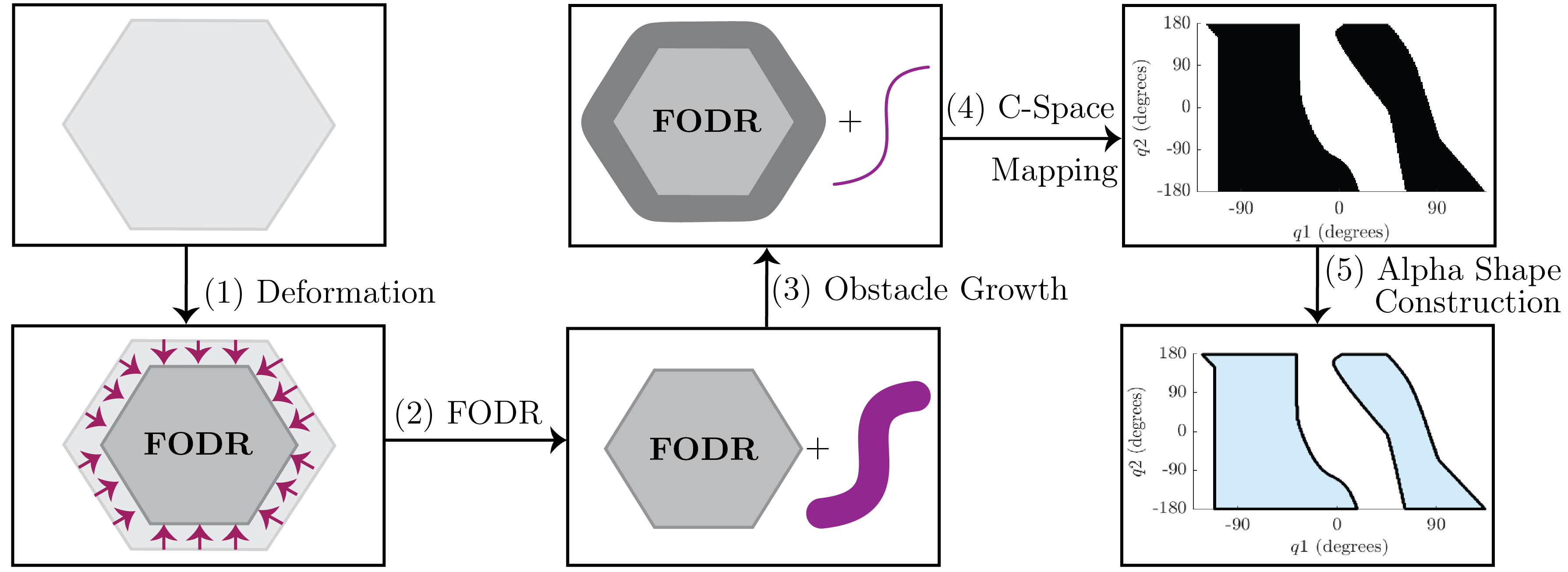}
    \vspace{-0.5cm}
    \caption{Overview of our force-safe environment maps and real-time detection approach. (1) illustrates the deformation of a polygonal environment along its inward normal direction to define the force-unsafe deformation region (2). (3) shows inflation of the FODR obstacles by the manipulator’s thickness, allowing us to consider only the backbone for subsequent computations. (4) demonstrates our method for mapping FODR poses to configuration space and (5) shows computation of alpha shape regions for efficient force safety queries for multiple robot configurations.}    \label{fig:overviewfigure2}
\end{figure*}

\section{Force-Unsafe Obstacle Maps in Task Space}
\label{sec:force_safety_task}
We adopt the notion of \emph{force safety} as
introduced in \cite{Dickson2025Safe}, where safety is not defined by
zero collisions with the environment but by environmental collisions that are bounded by maximum force thresholds.

\begin{definition}\emph{Force Safety:}
Following \cite{Dickson2025Safe}, a soft robot manipulator's motion is said to be \emph{force-safe} if its applied contact forces
$\bF(t)$ remain within a known admissible set $F_{\mathrm{safe}}$ for all time,
i.e.,
\[    \bF(0) \in \mathcal{F}_{\mathrm{safe}} \;\;\Rightarrow\;\; \bF(t) \in \mathcal{F}_{\mathrm{safe}}, \quad \forall t \geq 0.\]
Equivalently, force safety is achieved when $\mathcal{F}_{\mathrm{safe}}$ is an invariant
set of the system dynamics.
\end{definition}

To connect geometry with force limits, we adopt a model of the
deformable obstacle environment as a polytope $\mathcal{N} = \{ \br \mid H\br \leq \bh \}$ representing the original obstacle region (Fig. \ref{fig:shrinking_obstacles}). Each facet of this polytope is assumed to deform
elastically in its inward normal direction, with a
force–displacement relationship $\psi(\cdot)$. A maximum admissible contact
force $F^{max}$ then corresponds to a maximum deflection $n^{max} =
\delta^{-1}\psi^{-1}(F^{max})$, where $\delta \in [0,1]$ is a safety factor. Translating each facet of $\mathcal{N}$ inward by $n^{max}$
yields a new polytope
$\mathcal{P} = \{ \br \mid H'\br \leq \bh' \},$
which represents the set of robot poses that are guaranteed to apply more than $F^{max}$ to the obstacles in the environment. Here,  
$H' = H$, and $\mathbf{h}' = \mathbf{h} - n^{max} \sqrt{m^2+1}$ where $m$ is the slope of each hyperplane that forms the polytope $\mathcal{P}$. As shown
in \cite{Dickson2025Safe}, this translation provides a conservative but
tractable approximation of force safety in task space.

The set $\mathcal{P}$ can be interpreted as a \emph{force-unsafe obstacle deformation region} in task
space as seen in Fig. \ref{fig:shrinking_obstacles}. This construction encodes both geometry and allowable contact
forces. Therefore, the boundary of $\mathcal{P}$ distinguishes between robot poses that
are (i) collision-free, (ii) in contact but force-safe, and (iii) unsafe due to
exceeding $F^{max}$. This task-space representation forms the foundation for
our subsequent mapping into configuration space, where we construct
\emph{force-safe configuration space obstacle maps} for motion planning.

\begin{figure}[!b]
\centering    
\includegraphics[width=1.0\columnwidth]{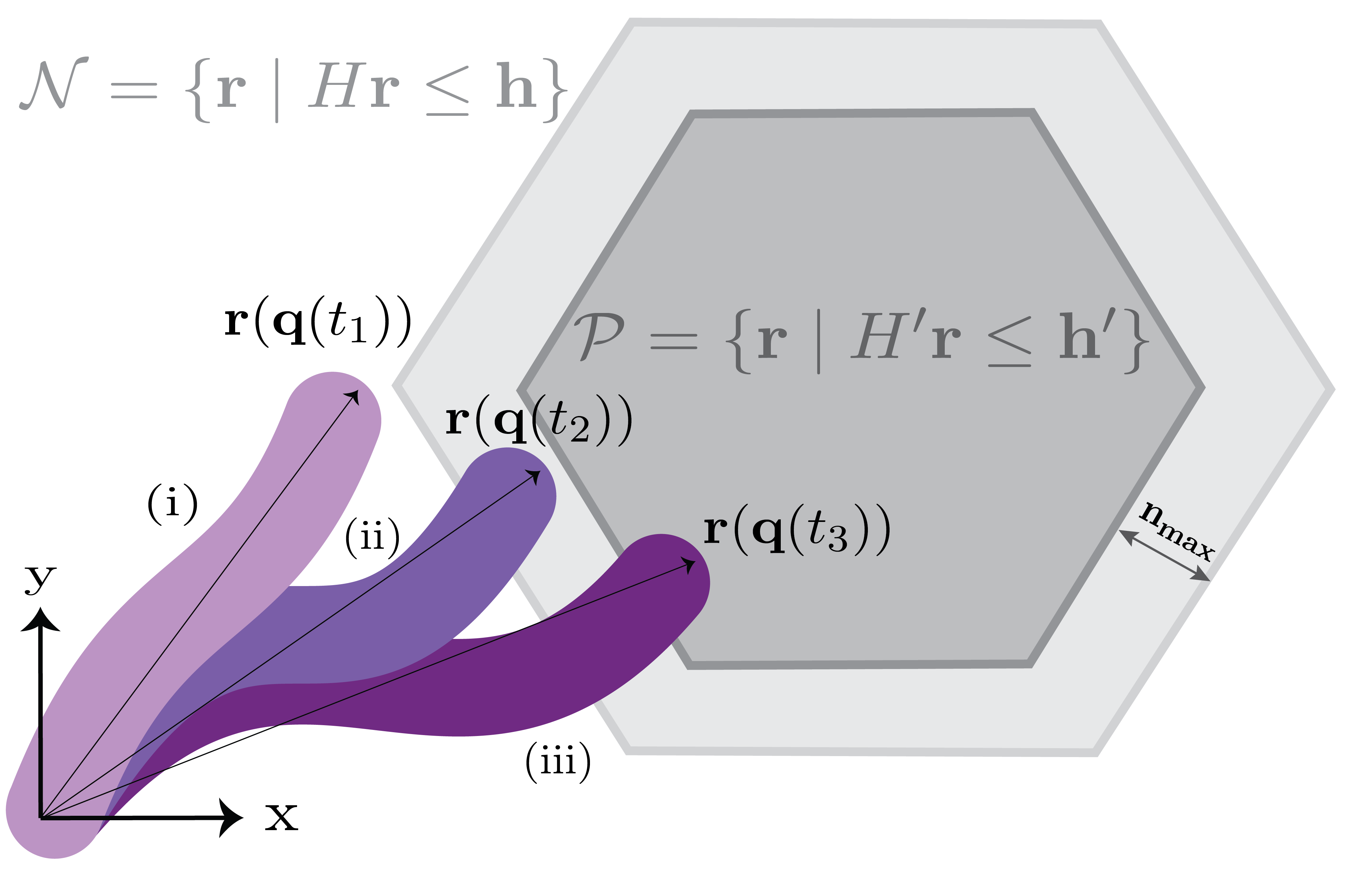}
\vspace{-0.3cm}
\caption{We define the original obstacle as polytope $\mathcal{N}$ and the force-unsafe deformation region as $\mathcal{P}$. Robot poses illustrate (i) collision-free (no contact) configuration, (ii) contact within force-safe limits, and (iii) unsafe contact exceeding $F^{max}$.}
\label{fig:shrinking_obstacles}
\end{figure}

\noindent We consider this definition of force safety under the following mild assumptions. 
\begin{itemize}
    \item The mathematical relationship between force and displacement is $\mathbf{F}_i(n_i) = \delta\psi_i(n_i) \mathbf{n}_i$ as we only consider normal forces.
    \item The  function $\psi_i(\cdot) : \mathbb{R} \mapsto \mathbb{R}$ that maps force to displacement is strictly monotonic for $n_i > 0$ and is $\psi = 0$ for $n \leq 0$.
\end{itemize}

\begin{definition}
The \emph{Force-Unsafe Obstacle Deformation Region ($\mathrm{FODR}$)}, 
is defined as the subset of task space corresponding to obstacle deformations that exceed the admissible force  threshold \(F^{max}\), i.e.,
$\mathrm{FODR} = \mathcal{P}$.
Hence, configurations whose  poses lie on or within $\mathrm{FODR}$ correspond to contacts that produce forces greater than or equal to $F^{max}$, i.e.,
$\br \in \mathrm{FODR} \;\;\Rightarrow\;\; \bF(\br) \geq F^{max}.$
\end{definition}

\begin{figure}[!t]
\centering    
\includegraphics[width=1.0\columnwidth]{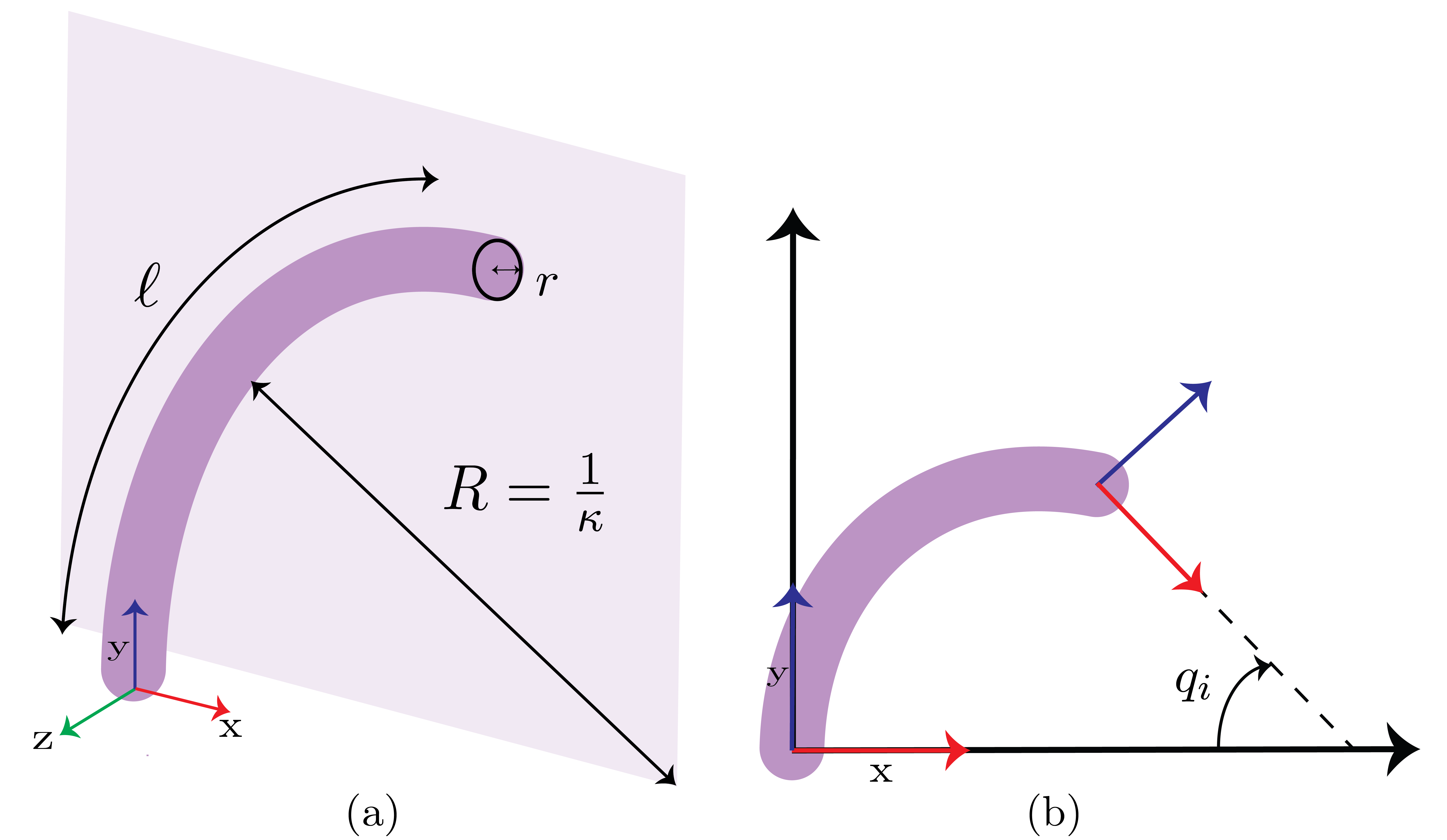}
\vspace{-0.4cm}
\caption{Piecewise Constant Curvature (PCC) kinematics of the soft robot manipulator \cite{Webster2010Design}. (a) 3D schematic of a single segment, parameterized by arc length $\ell$, curvature $\kappa$, and thickness $r$, with radius of curvature $R = 1/\kappa$. (b) Planar view showing the joint angle and backbone configuration of a single segment.}
\label{fig:softforwardkinematics}
\end{figure}

\section{Configuration Space Force-Safe Environment Maps and Real-Time Detection Approach}
In this section, we map the task space $\mathrm{FODR}$ into the robot’s configuration space and develop an algorithm for real-time force safety detection. This mapping is necessary because motion plans and controllers are designed in configuration space, where the robot's state space, control inputs and actuation limits are defined. 

\subsection{Soft Robot Manipulator Forward Kinematics}
\label{sec:softmanipulatorforwardkinematics}
For our proof-of-concept, we use the piecewise constant curvature (PCC) kinematics of the soft robot manipulator to map task space force safety into configuration space. 
The geometry and kinematics of a PCC soft robot manipulator allows us to determine the pose of points along its backbone. 
We consider an $n$ segment soft robot under the PCC assumption. Here, each $q_i$ represents the joint angle of the corresponding segment $i$ with an arc length or $L_i$. Under the piecewise constant curvature (PCC) assumption, each segment of the soft manipulator bends along a circular arc of fixed curvature as shown in Fig. \ref{fig:softforwardkinematics}. Let the $i$-th segment have arc length $L_i$, total bend angle $q_i$, curvature
$\kappa_i = \frac{q_i}{L_i}$, and 
radius of curvature $R_i = 1/\kappa_i = L_i/q_i$.

We parameterize points along the segment by the arc-length variable $s\in[0,L_i]$. Let
$\mathbf{r}_i(s) =
\begin{bmatrix}
x_i(s)& y_i(s)& z_i(s)
\end{bmatrix}^\top
\in\mathbb{R}^3$
denote the position of the point at arc-length $s$ along the centerline of segment~$i$. 
We represent the pose of the points along the circular backbone of the robot manipulator in homogeneous coordinates:
$\mathbf{P}_i(s)
=
\begin{bmatrix}
\mathbf{r}_i(s)&1
\end{bmatrix}^\top \in\mathbb{R}^4$
Each segment’s backbone is discretized as
$\{\mathbf{P}_i(s_j)\}_{j=1}^{N_s}$.

\begin{figure}[!t]
\centering    
\includegraphics[width=1.0\columnwidth]{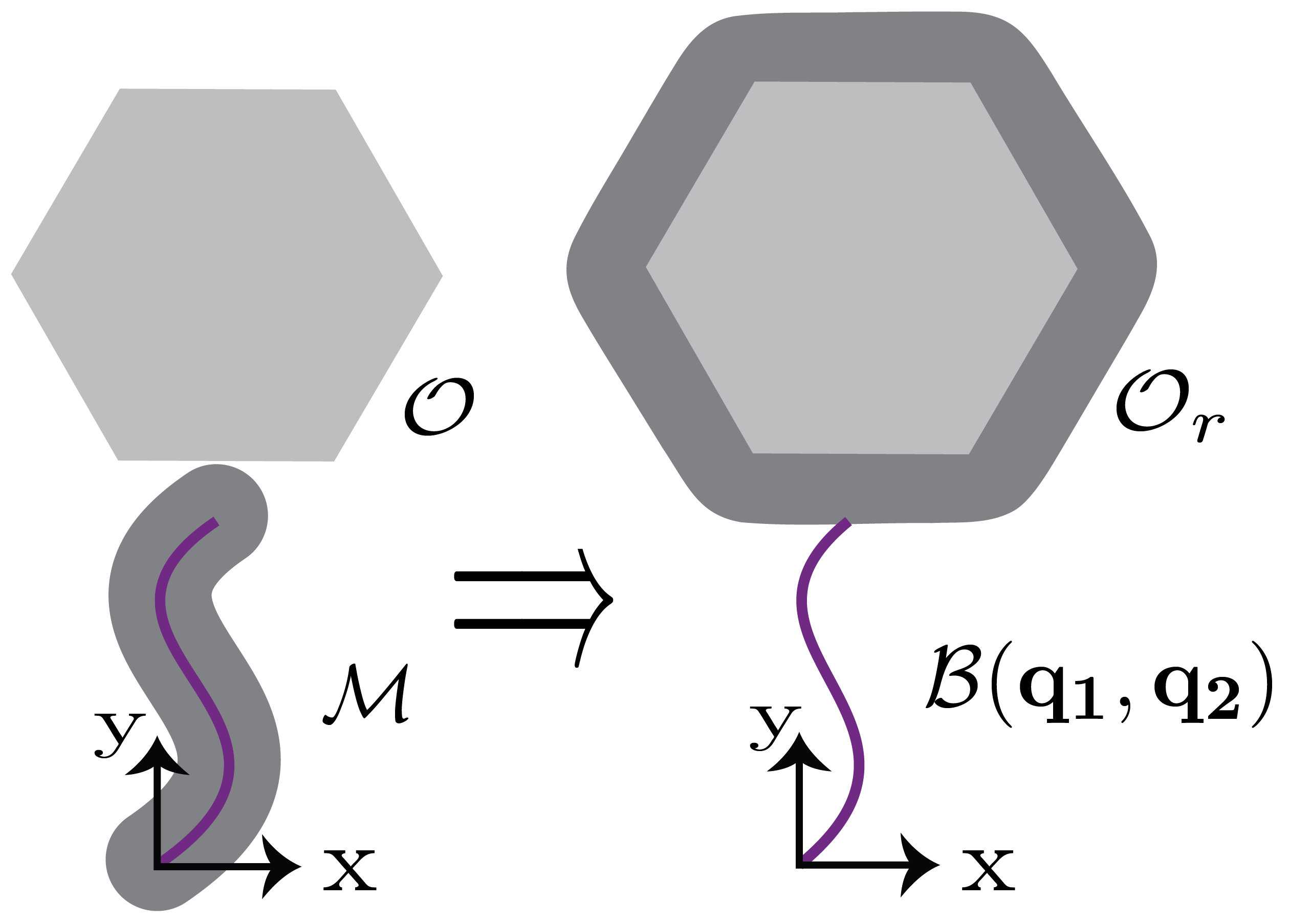}
\vspace{-0.3cm}
\caption{$\mathrm{FODR}$ inflation using Minkowski sums. (Left) The $\mathrm{FODR}$ (light gray) and soft robot manipulator are shown. (Right) The $\mathrm{FODR}$ boundary is inflated by the manipulator’s thickness (dark gray), resulting in a grown obstacle that accounts for the robot’s thickness}
\label{fig: obstacle+robotthickness}
\vspace{-0.3cm}
\end{figure}

The robot's kinematics are defined by a series of $i$ transformation matrices $T_0^1, \cdots, T_{n-1}^n$ which map each reference frame to a successive one. The transformation matrix $T_{i-1}^i$ is given by:
$\begin{bmatrix}
\mathbf{R}_i & \mathbf{t}_i \\
\mathbf{0}^{1\times3} & 1
\end{bmatrix}$,
where $\mathbf{R}_i\in SO(3)$ is the rotation matrix and $\mathbf{t}_i\in\mathbb{R}^3$ is the translation vector.

\subsection{Designing Task Space Obstacle Maps}
We expand each $\mathrm{FODR}$ in  task space by a margin equal to the manipulator’s thickness $r$. Intuitively, this ensures that if the centerline of the manipulator stays outside this expanded region, then the entire manipulator also remains safe. Directly checking whether the manipulator’s full geometry intersects each $\mathrm{FODR}$ at every configuration would be computationally expensive, since it requires computing the position and orientation of each segment through forward kinematics, then checking for intersections between the robot's thickened segments and the $\mathrm{FODR}$ at every time step.

\begin{algorithm}[]
\LinesNumbered
\KwIn{Set $\{\mathrm{FODR}_1, \dots, \mathrm{FODR}_N\}$ of $N$ task-space FODRs, and manipulator $\mathcal{M}$ with thickness $r$}
\KwOut{Grown obstacle $\mathcal{O}_r \in \mathbb{R}^3$ that accounts for 
the manipulator's thickness $r$}
Take the union of all FODRs:
$\mathcal{O} = \bigcup_{i=1}^N \mathrm{FODR}_i$\;
Grow $\mathcal{O}$ by the manipulator's thickness $r$ using Minkowski sums:
$\mathcal{O}_r = \mathcal{O} \oplus \mathcal{M} 
= \bigl\{\, \mathbf{o} + \mathbf{m} \;:\; 
\mathbf{o} \in \mathcal{O},\; \mathbf{m} \in \mathcal{M} \,\bigr\}$
\newline where $\mathcal{M} = \{\mathbf{m} \in \mathbb{R}^3 \;:\; \|\mathbf{m} - \mathbf{b}\| \le r, \forall \ \mathbf{b} \in \mathcal{B} \}$
\caption{Obstacle Growth Algorithm }
\label{alg:minkowski_sum}
\end{algorithm}

\noindent Instead, we achieve the same effect geometrically by growing the obstacle regions using Minkowski sums (Alg. \ref{alg:minkowski_sum}). This operation constructs a grown obstacle region that incorporates the manipulator’s thickness $r$ as shown in Fig. \ref{fig: obstacle+robotthickness}.
Geometrically, $\mathcal{O}_r$ represents the region of all points lying within a distance $r$ of the original $\mathrm{FODR}$ boundaries.
In other words, if the manipulator’s \emph{centerline} passes through any point of $\mathcal{O}_r$, then the body of radius $r$ surrounding that point would intersect $\mathcal{O}$. 
Hence, testing the centerline $\mathcal{B}(\mathbf{q})$ against $\mathcal{O}_r$ is equivalent to testing the full robot body against the original $\mathrm{FODR}$ obstacles $\mathcal{O}$.

\begin{figure*}[!t]
    \centering
    \vspace{0.2cm} \includegraphics[width=1.0\textwidth]{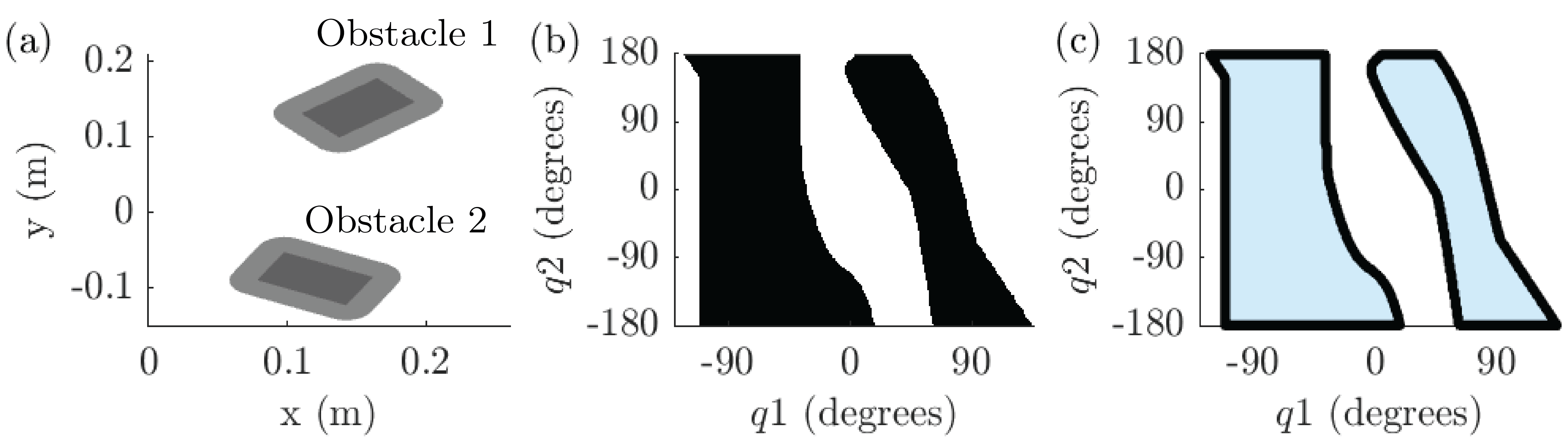}
    \vspace{-0.5cm}
    \caption{Task space (a) and corresponding configuration space (b) obstacle maps obtained from Alg. \ref{alg:task_to_cspace}. Grown task space obstacles are shown for easy reference; configuration space mapping identifies force-unsafe (black) and force-safe (white) joint configurations. Computed
    alpha shape representation (c) of force-unsafe regions in configuration space for efficient force-safety queries during planning.}
    \label{fig:taskspace_map}
    \vspace{-0.2cm}
\end{figure*}

\subsection{Real-Time Force-Safe Detection Algorithm}
Accurate real-time detection of force-unsafe configurations is critical for preventing excessive environmental contact forces during manipulation. Our goal is to determine, for any given configuration $\mathbf{q}$, whether the manipulator violates maximum force thresholds. 
To achieve this efficiently, we formulate a force-safe detection algorithm that relies on the manipulator’s kinematic model to map force-safe task space poses to their corresponding manipulator joint angles in configuration space. 

\begin{algorithm}[]
\caption{Real-Time Force-Safe Detection Algorithm}
\label{alg:task_to_cspace}
\KwIn{
Grown obstacle set $\mathcal{O}_r \subset \mathbb{R}^3$, 
segment arc lengths $\{L_i\}_{i=1}^{n}$, 
and discretized joint grids 
$\mathcal{G}_i = \{q_{i,k_i}\}_{k_i=1}^{N_i}$ for each segment $i$.}
\KwOut{Configuration-space obstacle array $C_{\text{obs}}[k_1,\dots,k_n]$}
\For{each $\mathbf{q} = (q_{1,k_1}, \dots, q_{n,k_n}) \in \mathcal{G}_1 \times \cdots \times \mathcal{G}_n$}
{ Initialize backbone: $\mathcal{B}(\mathbf{q}) \gets \varnothing$ 
\newline
\For{$i = 1$ to $n$}
    {
        \For{$s = 1$ to $N_s$}
        {
            Compute local backbone point: $\mathcal{B}_i(q_{i,k_i}, s)$
             Append to full backbone: $\mathcal{B}(\mathbf{q}) \gets \mathcal{B}(\mathbf{q}) \cup \mathcal{B}_i(q_{i,k_i}, s)$
        }
    }
     \textbf{If} $\mathcal{B}(\mathbf{q}) \cap \mathcal{O}_r \neq \emptyset$ \textbf{then} mark $\mathbf{q}$ as force-unsafe.
     Set $C_{\text{obs}}[k_1,\dots,k_n] \gets 1$ if intersection occurs, else $0$.}
\end{algorithm}

To obtain a global map of force-safe and unsafe configurations, we discretize the configuration space into uniform joint grids:
$q_i \in \mathcal{G}_i = \{q_{i,k_i}\}_{k_i=1}^{N_i}$.
At each grid point, the backbone is compared against $\mathcal{O}_r$. 
The output of this procedure is a configuration-space obstacle matrix
$C_{\text{obs}}[k_1,\dots,k_n] = \chi(\mathbf{q})$. The overall procedure is highlighted in Alg~\ref{alg:task_to_cspace}, which efficiently constructs the environment map.

\begin{theorem}[Force-Safe Configuration Detection]
\label{thm:forcesafe}
Given the grown force-unsafe regions $\mathcal{O}_r$, which accounts for the manipulator’s finite thickness,
a configuration $\mathbf{q}$ of the soft manipulator is \emph{force-safe} if 
$\mathcal{B}(\mathbf{q}) \cap \mathcal{O}_r = \emptyset$,
where $\mathcal{B}(\mathbf{q})$ denotes the manipulator’s backbone obtained from the PCC kinematic model.
Equivalently, the force-unsafe detection function
\begin{equation}
\chi(\mathbf{q}) =
\begin{cases}
1, & \exists\, \mathbf{b} \in \mathcal{B}(\mathbf{q}) \text{ such that } \mathbf{b} \in \mathcal{O}_r, \\
0, & \text{otherwise,}
\end{cases}
\end{equation}
serves as an indicator for force-unsafe configurations, where $\chi(\mathbf{q}) = 0$ implies the manipulator is force-safe.
\end{theorem}
\begin{proof}
\[
\begin{aligned}
\chi(\mathbf{q}) = 0
& \implies  \forall \mathbf{b} \in \mathcal{B} (\mathbf{q}), \; \mathbf{b} \notin \mathcal{O}_r \\ & \implies  \mathbf{F}(\mathbf{q}) \in \mathcal{F}_{\mathrm{safe}} \;\; \forall t \ge 0 
\end{aligned}
\]
\end{proof}
The matrix $C_{\text{obs}}$ provides a discrete yet topologically accurate representation of the manipulator's force-unsafe configurations defined in Theorem \ref{thm:forcesafe} (Fig. \ref{fig:taskspace_map} b). However, practical use in real-time planning and control requires a continuous representation of this mapping that can be queried at arbitrary configurations.

\subsection{Constructing the Continuous Force-Unsafe Set via $\alpha$-Shape Reconstruction}
Although the occupancy grid $C_{\text{obs}}$ is a set of sampled force-unsafe configurations, it is inherently discrete i.e. it only provides force safety information at the grid points and cannot be directly queried for arbitrary values of $\mathbf{q}\in \mathbb{R}^n$. In practice, a manipulator’s motion may pass through configurations that lie between grid points, and without a continuous representation, it would be unclear whether such intermediate configurations are force-safe. To address this limitation and enable computationally efficient, real-time force safety queries for motion planning, we reconstruct a continuous geometric approximation of the force-unsafe region in configuration space (Fig. \ref{fig:taskspace_map} c). 
This continuous set $\mathcal{F}_{\mathrm{unsafe}}$,
is derived from the discrete occupancy grid using a shape reconstruction procedure based on $\alpha$-shape theory.

\subsubsection{Connected Components in the Occupancy Grid}
We first extract the connected components of $C_{\text{obs}}$, identifying clusters of adjacent grid cells labeled as unsafe. 
Let $\{\mathcal{C}_k\}_{k=1}^{N_c}$ denote these components. 
Each component $\mathcal{C}_k$ is represented as a point cloud in the continuous configuration space map:
$S_k = \bigl\{(q_{1,k_1}, \dots, q_{n,k_n}) \;:\; (k_1,\dots,k_n) \in \mathcal{C}_k\bigr\}
$. 
The union of all the sampled force-unsafe sets $S = \bigcup_{k=1}^{N_c} S_k$ 

\subsubsection{Continuous Reconstruction via $\alpha$-Shapes}
The $\alpha$-shape generalizes the concept of a convex hull i.e. a convex hull provides the simplest enclosing polygon, while the $\alpha$-shape allows the boundary to follow the shape of the map more closely by controlling a scale parameter $\alpha$. In this work, for each component $S_k$, we compute the vertices of the enclosing polygon using the algorithm described in \cite{bernardini1997sampling, guo1997surface}. This efficiently determines the ordered set of points $\{\mathbf{v}_1, \dots, \mathbf{v}_m\}$ defining the polygonal boundary. MATLAB’s built-in \texttt{alphaShape} function implements this procedure, yielding the $\alpha$-shape $P_k$ that approximates the force-unsafe region corresponding to $S_k$.
Each resulting $\alpha$-shape $P_k$ is a closed, simple polygon.

\subsubsection{Definition of the Force-Unsafe Set}
The union of all reconstructed polygons  $P_k$ forms the continuous force-unsafe region in configuration space:
$\mathcal{F}_{\mathrm{unsafe}} = \bigcup_{k=1}^{N_c} P_k
\subset \mathbb{R}^n.$
A configuration $\mathbf{q}$ is force-unsafe if it lies within the polygon $\mathcal{F}_{\mathrm{unsafe}}$. Since $\mathcal{F}_{\mathrm{unsafe}}$ is stored as a collection of polygons, safety queries are simplified to standard point-in-polygon membership tests, enabling evaluation of the force-safety condition in real-time.

\subsection{Implementation of the Real-Time Force-Safe Detection Algorithm}
We implement, test, validate and visualize our algorithm in 2D with $n=2$ segments in order to illustrate concepts clearly. To implement Alg.~\ref{alg:task_to_cspace}, we discretize the configuration space over the full joint range $[-180^\circ, 180^\circ]$ with a resolution of $1^\circ$ per joint ($\Delta q_1 = \Delta q_2 = 1^\circ$). This resolution offers a favorable trade-off between accuracy and computational efficiency. 
Each grid configuration $(q_{1,k_1}, q_{2,k_2})$ is assigned a precomputed binary label $C_{\text{obs}}[k_1,k_2] \in \{0,1\}$ indicating whether the corresponding manipulator configuration is force-unsafe. To construct a continuous approximation of the force-unsafe set, we generate an $\alpha$-shape from the discrete grid points. The initial scale parameter is chosen proportional to the grid resolution, $\alpha_0 = c_\alpha \Delta q$, where $c_\alpha>0$ ensures edges connect only neighboring points while avoiding oversmoothing across separate clusters. We iterate over several $\alpha$ values and select the smallest that yields a valid, non-self-intersecting polygon. This process produces a family of disjoint polygons $\{P_k\}$ representing continuous approximations of each discrete component $S_k$. Then, we form a union of all reconstructed polygons $\mathcal{F}_{\mathrm{unsafe}}$. Finally, force safety for any configuration is evaluated by sampling points along the manipulator backbone and testing for intersections with $\mathcal{F}_{\mathrm{unsafe}}$, enabling efficient, real-time assessment.

\section{Simulation Setup and Results}
Our proposed approach discussed in Alg. \ref{alg:task_to_cspace} is validated in simulation using joint-space trajectories of our planar two-segment soft robot manipulator. The results demonstrate the correctness of the computed configuration space obstacle map (Alg. \ref{alg:task_to_cspace}) in validating force safety.

\subsection{Simulation Setup}
In this manuscript, to evaluate our proposed approach under representative but uncontrolled open loop motion, we employ an open-loop nominal controller that generates smooth, continuous actuation using sinusoidal  inputs. Specifically, each sinusoidal input follows: $A_j\sin(2\pi f t)$. We achieve the desired coverage by tuning the amplitude and frequency of the sine waves to be  $A_j=160$ hPa (hecto-Pascal) and $f=0.0071$ Hz. With this, we achieve a coverage of approximately  $\pm 30^\circ$. The sinusoidal signals are designed to span the robot's state space in a repeatable yet dynamic manner enabling us to evaluate our real-time force safety detection approach. The manipulator we use and calibrate in simulation consists of two planar continuum segments of equal length of $0.122 m$. Each segment is discretized as a circular arc with 150 points for the backbone computation (i.e. $N_s = 150$). The thickness or radius $r$ of each segment is $0.02 m$.

To reproduce the hardware setup faithfully, the geometry and placement of the simulated obstacles were calibrated using the actual positions of the force plates in the hardware setup (Fig. \ref{fig:hardware_setup} below). The obstacle planes in simulation were aligned to match the spatial location and orientation measured from the overhead camera’s pixel data of the hardware workspace. Fiducial markers (AprilTags) attached to the force plates provided reference points, allowing the pixel coordinates of the plate edges to be converted into metric workspace coordinates through camera calibration. These coordinates were then used to position the obstacle polygons in the simulation so that their boundaries corresponded precisely to the real force plate locations.

\begin{figure*}[!t]
    \centering
    \vspace{0.2cm} \includegraphics[width=1.0\textwidth]{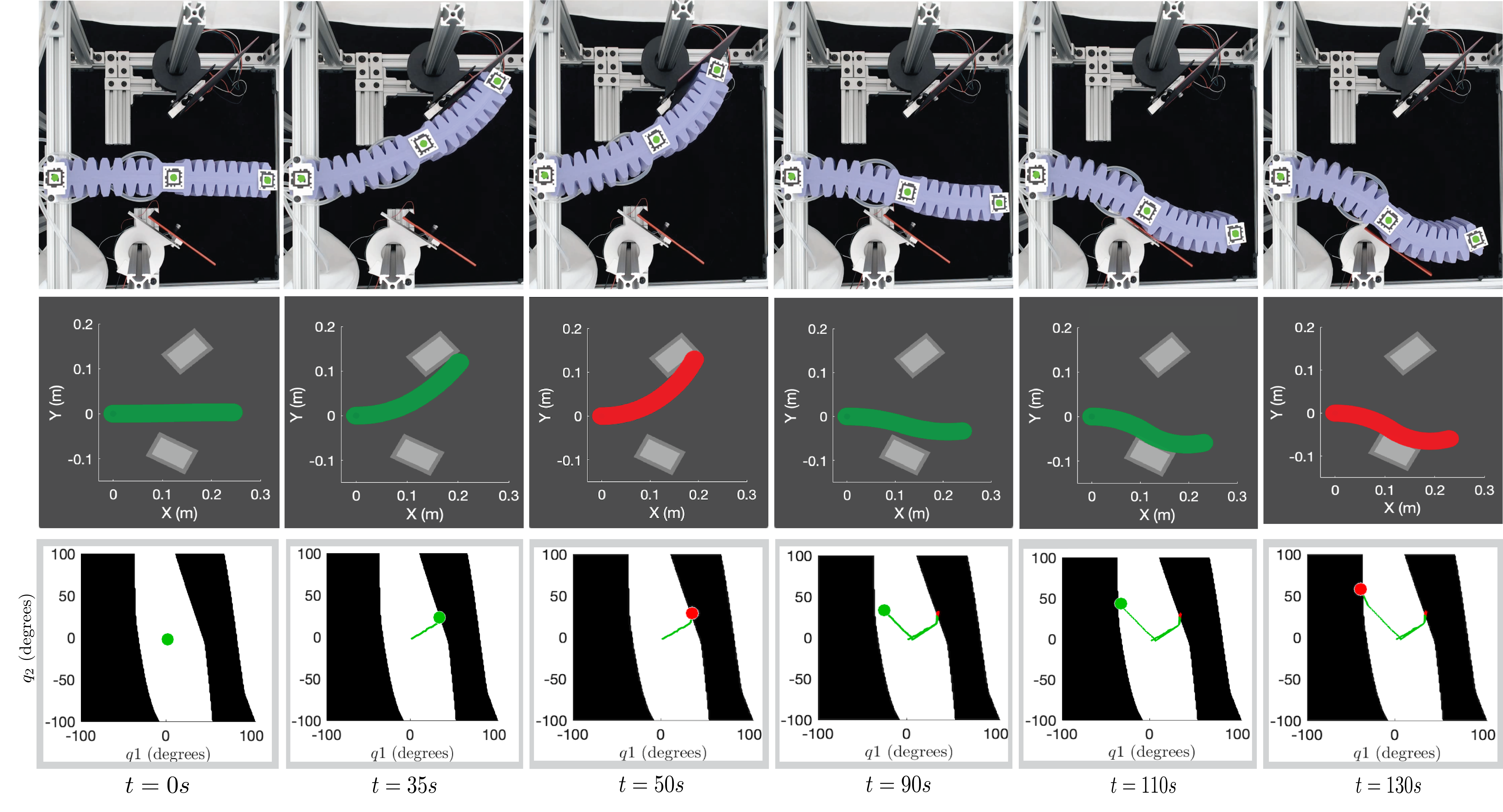}
    \vspace{-0.5cm}
    \caption{Comparison of simulation and hardware results at six time steps, showing equivalent robot configurations in task space (top row: hardware, middle row: simulation) and configuration space (bottom row: simulation). Green and red colors denote force-safe and force-unsafe contact, respectively, as detected by the force safety detection framework. In simulation, we show both original obstacle polytope $\mathcal{N}$ (dark gray) and $\mathrm{FODR}$ region (light gray). This visualization highlights agreement between hardware and simulation at each time step for force safety assessment.}
    \label{fig:sim_hardware_timelapse}
    \vspace{-0.3cm}
\end{figure*}

We compute Alg. \ref{alg:task_to_cspace} to obtain our force-unsafe obstacle map $C_{\text{obs}}$ in configuration space (Fig. \ref{fig:taskspace_map} b). We place two obstacles in the manipulator's environment and set the maximum allowable force $F^{max}= 0.105N$ with a corresponding stiffness coefficient of each obstacle $k_{env} = 11.16 N/m$ and a safety factor $\delta$ of $0.95$. At each time step, we query the configuration $q_1(t), q_2(t)$ against the precomputed obstacle polygons $\mathcal{F}_{\mathrm{unsafe}}$ that are obtained from our algorithm as shown in Fig. \ref{fig:taskspace_map} c. Furthermore, we compute the simulated force exerted on each obstacle by the soft robot manipulator. For each configuration $q_1(t), q_2(t)$, the backbone is reconstructed as a series of points $s$ along the circular arc of the two segments. Each point's distance to the FODR's plane is evaluated, accounting for both the finite length of the plane and the manipulator thickness $r$. The local force at each point is then computed using the force-displacement model: $F_k(t) = k_{env}\max(0, r - n_k(t))$
where $n_k(t)$ is the signed distance from the $k$-th backbone point to the plane. We compute the simulated force for each configuration as the maximum of all the forces exerted at each of the points along the backbone $F(t) = \max_k F_k(t)$.

\begin{figure}[!b]
\centering    
\includegraphics[width=1.0\columnwidth]{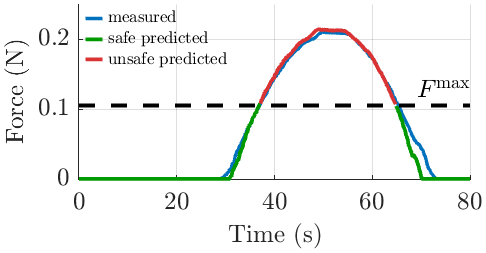}
\vspace{-0.5cm}
\caption{Simulated contact forces for Obstacle 1. Green and red data points indicate force-safe and force-unsafe predictions, respectively, relative to the threshold $F^{max}$(black dashed line). Hardware force measurements (blue) are shown for reference.}
\label{fig:forceplate1}
\vspace{-0.5cm}
\end{figure}

\subsection{Simulation Results}
A simulated animation was generated by rendering our soft robot backbone with thickness $r$ and obstacle geometry at each time step. The soft robot manipulator was colored green when $(q_1(t),q_2(t) \notin \mathcal{F}_{\mathrm{unsafe}}$ and red when $(q_1(t),q_2(t) \in \mathcal{F}_{\mathrm{unsafe}}$. We show a timelapse of this simulation in Fig. \ref{fig:sim_hardware_timelapse}. The timelapse verifies that configurations corresponding to force-unsafe deformation of the obstacle align with a red-colored soft robot manipulator. The manipulators transition into the red regime thereby confirms that the precomputed polygons accurately capture environmental contact force safety thresholds. For instance, we observe at $t = 35s$ that the manipulator deforms the first obstacle but does not exceed the maximum allowable force. However, at $t=50s$, the manipulator enters the $\mathrm{FODR}$ region thus exceeding the maximum force (Fig. \ref{fig:sim_hardware_timelapse}). 
Each force safety query takes an average of $0.000223s$, which meets the requirements for real-time execution. The simulated force exerted by the manipulator on both obstacles in its environment at each time-step is computed and plot over time as shown in Fig. \ref{fig:forceplate1} and Fig. \ref{fig:forceplate2}. 
\begin{figure}[!b]
\centering    
\includegraphics[width=1\columnwidth]{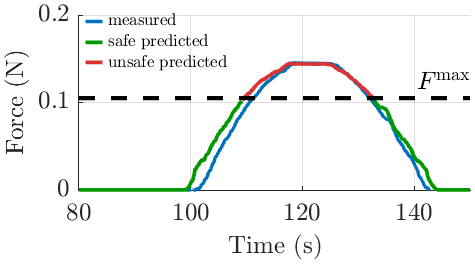}
\vspace{-0.5cm}
\caption{Simulated contact forces for Obstacle 2. Green and red data points indicate force-safe and force-unsafe predictions, respectively, relative to the threshold $F^{max}$(black dashed line). Hardware force measurements (blue) are shown for reference.}
\label{fig:forceplate2}
\end{figure}
Time-series plots are color-coded; green indicates force-safe configurations, while red indicates predicted force-unsafe contact. These results and visualizations confirm that the computed FODR regions correspond to the regions of force unsafety at each time-step. The simulated force plots therefore provide both a quantitative measure of interaction force magnitudes and a visual correspondence to the C-space unsafe regions, enabling thorough validation of our real-time force safety detection approach.

\section{Hardware Setup and Results}
To further validate our  approach, we develop a hardware experimental platform. Our results show that our method accurately detects poses that result in unsafe forces. 

\subsection{Hardware Setup}
The experimental platform (Fig. \ref{fig:hardware_setup}) consists of a planar soft pneumatic limb composed of two segments, each actuated by a pair of antagonistic air chambers molded from silicone elastomer. Differential pressurization of opposing chambers produces planar bending. Chamber pressures 	
are regulated by proportional valves (iQ Volta) under low-level PID control implemented on an Arduino microcontroller. Pressure readings are provided by pressure sensors (Honeywell MPRLS), and air is supplied from an in-house fabricated reservoir whose pressure is maintained via another microcontroller-driven PID loop controlling an air pump (Adafruit 4700). This closed-loop pneumatic infrastructure allows for stable pressure regulation and reproducible actuation during experiments.

The robot configuration is tracked using an overhead camera equipped with fiducial markers (AprilTags) placed at the base and end of each segment. The image processing pipeline, implemented in OpenCV, estimates the position of each marker and computes the corresponding joint angles using geometric relationships between tags. The camera provides an average update rate of approximately 12.8 fps.
\begin{figure}[!t]
\centering    
\includegraphics[width=1.0\columnwidth]{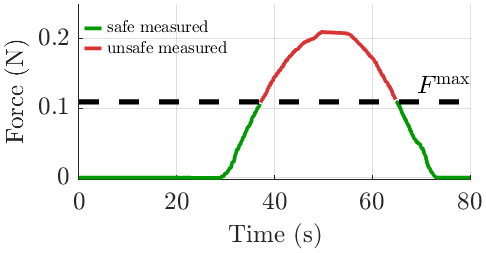}
\vspace{-0.5cm}
\caption{Measured contact forces from force plate 1 during the hardware experiments. Green and red data points indicate force-safe and force-unsafe values, respectively, relative to the threshold $F^{max}$ (black dashed line).}
\label{fig:measuredforceplate1}
\end{figure}
Environmental contact forces are measured using flexible ABS sheets attached to force gauge sensors that serve as deformable obstacles. Each plate has known stiffness $k_{env}$. To calibrate this stiffness, known masses are applied to the plate while deflections are measured, enabling computation of the spring constant from the linear force–deflection relationship. During operation, the force gauge measures the normal contact force applied by the manipulator. We then apply the proposed force safety detection algorithm (Alg. \ref{alg:task_to_cspace}) to identify configurations where the robot’s interaction with the environment exceeds the allowable force threshold, thereby indicating force-unsafe states.

\begin{figure}[!t]
\centering    
\includegraphics[width=1.0\columnwidth]{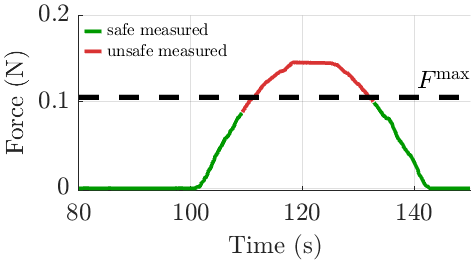}
\vspace{-0.5cm}
\caption{Measured contact forces from force plate 2 during the hardware experiments. Green and red data points indicate force-safe and force-unsafe values, respectively, relative to the threshold $F^{max}$ (black dashed line).}
\label{fig:measuredforceplate2}
\vspace{-0.5cm}
\end{figure}

\subsection{Hardware Experimental Results}
To validate the proposed framework in hardware, each configuration 
$\mathbf{q}$ was queried against the precomputed configuration-space obstacle set $\mathcal{F}_{\mathrm{unsafe}}$. 
Configurations within this set were labeled as force-unsafe, corresponding to robot poses that exerted  contact forces beyond the maximum threshold. Consequently, configurations outside this set were classified as force-safe.
The resulting classification was visualized by color-coding the manipulator; green indicates force-safe states, while red indicates predicted force-unsafe contact. The hardware experiments showed clear transitions between these two regimes as the manipulator approached and deflected the deformable plates as seen in Fig \ref{fig:measuredforceplate1} and \ref{fig:measuredforceplate2}.

\begin{figure}[!b]
\centering    
\includegraphics[width=0.8\columnwidth]{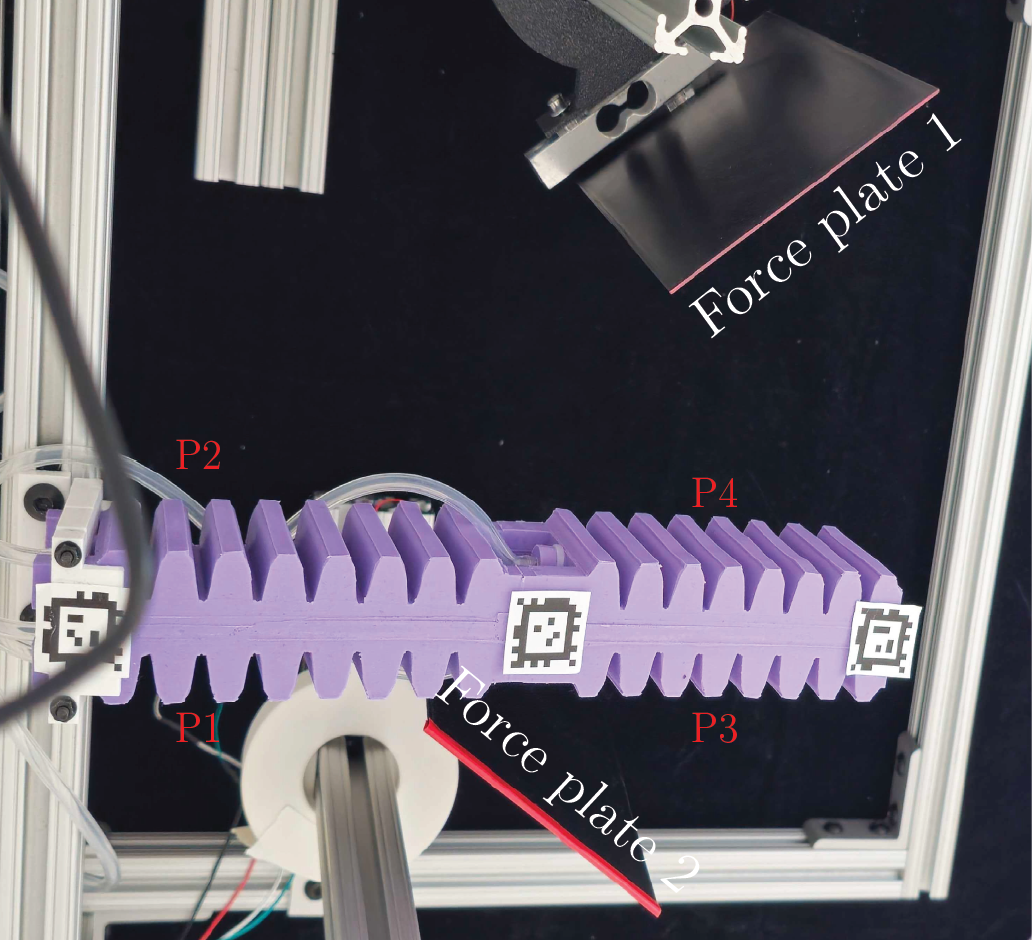}
\vspace{-0.3cm}
\caption{Hardware test setup with soft robot manipulator for force-safe real-time detection approach experimental validation. Robot actuation is achieved by increasing/decreasing the measured pressures ($p_1$, $p_2$, $p_3$, $p_4$) in the chambers of the segments.}
\label{fig:hardware_setup}
\end{figure}
We observe that periods of measured contact forces exceeding $F^{max}$ coincided precisely with red (force-unsafe) intervals in the color-coded force plots (Fig \ref{fig:measuredforceplate1} and \ref{fig:measuredforceplate2}), while periods of measured contact forces below $F^{max}$ coincided with green (force-safe) regions. These results confirm that the polygon $\mathcal{F}_{\mathrm{unsafe}}$ correctly represents the regions of excessive environmental contact forces.

\section{Discussion and Conclusion}
Results in both simulation and hardware confirm that our proposed approach accurately detects force-safe and unsafe states of a soft robot manipulator. Our experimental validation is currently limited to a 2D planar soft robot manipulator. While this restriction enables clear visualization and hardware implementation, it does not incorporate 3D motions. Notably, our theoretical framework is developed for 3D soft robot manipulators with $n$ many segments. Therefore, extending our hardware validation to 3D motions constitutes a technical implementation challenge rather than a theoretical limitation. 

Furthermore, our experiments assume PCC kinematics \cite{Webster2010Design},
a widely used kinematic modeling approximation in soft robotics. PCC kinematics models have been successfully implemented in multiple prior studies \cite{dickson2025real, Dickson2025Safe}, and our simulation and hardware results provide additional evidence of their practical fidelity. In addition, our approach can be adapted to alternative kinematic models by substituting the forward kinematics component.  

We also formulate our problem in a known environment with normal deformation of convex obstacles. This assumption allows for a clear interpretation of safety margins but does not generalize to dynamic, unknown or nonconvex environments commonly encountered in real-world tasks. While many effective strategies exist for handling environmental uncertainty and non convex environments, such as adaptive control, and convex decomposition of nonconvex obstacles \cite{dickson2024spline}, incorporating these into our framework remains an important direction for future work.

Overall, this work presents a framework for real-time force safety detection for soft robot manipulators. The proposed approach enables accurate force safety mapping and deformation detection, and lays a strong foundation for future extensions to force-safe motion planning and control for soft robot manipulators. Ultimately, this advances the goal of achieving force-safe, contact-aware manipulation in delicate real-world environments.

\bibliographystyle{ieeetr}
\bibliography{references}

\end{document}